\documentclass[12pt]{article}
\usepackage[margin=1.2in]{geometry}
\usepackage[utf8]{inputenc}

\usepackage[english]{babel}
\usepackage[round]{natbib}
\usepackage{authblk}
\usepackage[symbol]{footmisc}

\usepackage[dvipsnames]{xcolor}

\usepackage{amsmath, amssymb, amsthm,mathtools,dsfont}
\usepackage{bbm}
\usepackage{natbib}
\usepackage{blkarray}
\usepackage{cancel}
\usepackage{enumitem}
\usepackage{euscript}
\usepackage{float}
\usepackage{bm}
\usepackage{graphicx}
\usepackage{mathrsfs}
\usepackage{microtype}
\usepackage{ragged2e}
\usepackage{sectsty}
\usepackage{thmtools}
\usepackage{tikz}
\usepackage[Symbolsmallscale]{upgreek}

\usepackage{microtype}
\usepackage{graphicx}
\usepackage{subcaption}

\usepackage{hyperref}

\newcommand{\cN}{\mathcal{N}}
\newcommand{\cO}{\mathcal{O}}
\newcommand{\cP}{\mathcal{P}}

\newcommand{\cT}{\mathcal{T}}

\newcommand{\XX}{\mathbb{X}}
\newcommand{\YY}{\mathbb{Y}}

\newcommand{\kl}[2]{\text{KL}(#1 \| #2)}

\newcommand*{\triplenorm}[1]{{\left\vert\kern-0.25ex\left\vert\kern-0.25ex\left\vert #1
    \right\vert\kern-0.25ex\right\vert\kern-0.25ex\right\vert}}

\newcommand{\R}{\mathbb{R}}
\newcommand{\Rd}{\mathbb{R}^d}
\renewcommand{\phi}{\varphi}
\newcommand{\eps}{\varepsilon}
\newcommand{\sse}{\subseteq}

\newcommand*{\E}{\mathbb E}

\DeclareMathOperator{\tr}{tr}

\newcommand*{\defeq}{\coloneqq}
\newcommand*{\rd}{\mathrm{d}}
\newcommand*{\dd}{\, \rd}
\DeclareMathOperator*{\argmin}{argmin}

\newcommand{\OT}{\text{OT}}

\usepackage{algorithm}
\usepackage{algpseudocode}
\renewcommand{\tilde}{\widetilde}
\newcommand{\TCB}{T_{\rm{CB}}}

\usepackage{multirow}

\usepackage{fancyhdr}

\usepackage[capitalize,noabbrev]{cleveref}

\usepackage{todonotes}
\theoremstyle{plain}
\newtheorem{theorem}{Theorem}[section]

\newtheorem{prop}[theorem]{Proposition}
\newtheorem{lemma}[theorem]{Lemma}

\theoremstyle{definition}

\theoremstyle{remark}
\newtheorem{remark}[theorem]{Remark}

\hypersetup{citecolor=purple, urlcolor=cyan, colorlinks=true, linkcolor=blue}

\title{Conditional simulation via entropic optimal transport:\\ 
Toward non-parametric estimation of \\ conditional Brenier maps}

\author[1]{Ricardo Baptista\footnote{Contributed equally. Correspondence to {\tt{rsb@caltech.edu}} and {\tt{ap6599@nyu.edu}}}}

\newcommand\CoAuthorMark{\footnotemark[\arabic{footnote}]} 
\author[2]{Aram-Alexandre Pooladian\protect\CoAuthorMark}
\author[3]{Michael Brennan}
\author[3]{Youssef Marzouk}
\author[2,4]{Jonathan Niles-Weed}
\affil[1]{CalTech }
\affil[2]{Center for Data Science, New York University}
\affil[3]{Massachusets Institute of Technology} 
\affil[4]{Courant Institute of Mathematical Sciences, New York University}

\begin{document}
\vspace*{-0.5in}

\begin{center} {\LARGE{{Conditional simulation via entropic optimal transport:\\ Toward non-parametric estimation of conditional Brenier maps}}}

{\large{
\vspace*{.3in}
\begin{tabular}{cccc}
Ricardo Baptista$^{1,\star}$, Aram-Alexandre Pooladian$^{2,\star}$, Michael Brennan$^{3}$, \\
Youssef Marzouk$^{3}$, Jonathan Niles-Weed$^{2,4}$\\
\hfill \\
\end{tabular}
{
\vspace*{.2in}
\begin{tabular}{c}
				$^1$California Institute of Technology\\
				$^2$Center for Data Science, New York University \\
                $^3$Massachusetts Institute of Technology \\
                $^4$Courant Institute of Mathematical Sciences, New York University\\
\end{tabular} 
}

}}
\vspace*{.1in}

\today

\end{center}

\begin{abstract}
Conditional simulation is a fundamental task in statistical modeling: 
given a finite collection of samples from a joint distribution, it consists in generating samples from conditionals of this distribution.
One promising approach is to construct conditional Brenier maps, where the components of the map push forward a reference distribution to conditionals of the target.
While many estimators exist, few, if any, come with statistical or algorithmic guarantees.
To this end, we propose a non-parametric estimator for conditional Brenier maps based on the computational scalability of \emph{entropic} optimal transport. Our estimator leverages a result of \citet{carlier2010knothe}, which shows that optimal transport maps under a rescaled quadratic cost asymptotically converge to conditional Brenier maps; our estimator is precisely the entropic analogues of these converging maps.
We provide heuristic justifications for choosing the scaling parameter in the cost as a function of the number of samples by fully characterizing the Gaussian setting. We conclude by 
comparing the performance of the estimator to other machine learning and non-parametric approaches on benchmark datasets and Bayesian inference problems.\looseness-1
\end{abstract}
\footnotetext{$^\star$Equal contribution. Correspondance to {\tt{rsb@caltech.edu}} and {\tt{ap6599@nyu.edu}}.}

\section{Introduction}
Given access to i.i.d.~samples from a joint distribution $\mu \in \cP(\R^{d_1}\!\times\!\R^{d_2})$, our goal is to generate samples distributed according to the conditional distribution $\mu_{2|1}(\cdot|x_1) \defeq \mu(\cdot,x_1)/\mu_1(x_1)$ (where $\mu_1$ is the first marginal of $\mu$) for any $x_1 \in \R^{d_1}$. This sampling problem lies at the core of computational Bayesian inference, where the joint distribution is specified by a model for observations $X_1$ and parameters $X_2$. 
The goal of simulation-based Bayesian inference is to sample from the posterior distribution, i.e., the conditional distribution $\mu_{2|1}(\cdot|x_1)$ corresponding to an observation $x_1$, given samples from the joint distribution~\citep{cranmer2020frontier}.\footnote{Joint data in this setting are cheaply obtained  by sampling $x_2$ from its marginal (prior) distribution $\mu_2$ and $x_1$ from the specified likelihood model $\mu_{1|2}(\cdot|x_2)$.}

\paragraph{Sampling via transport.} One natural framework for performing conditional simulation uses \emph{measure transport} \citep{marzouk2016sampling} by seeking a transport map that pushes forward a known source distribution to the conditional $\mu_{2|1}(\cdot|x_1)$ for any realization of the conditioning variable $x_1$. 

Many transport approaches for conditional simulation obey the following construction: 
Let $\rho = \rho_1 \otimes \rho_2$ be a tensor-product source measure that is easy to sample from (e.g., the standard Gaussian) and consider 
a transport map $T$ from $\rho$ to $\mu$ of the form
\begin{align}\label{eq:cond_brenier}
    T(x) = \begin{bmatrix*}[l]
        T^1(x_1) \\
        T^2(x_2;x_1)
    \end{bmatrix*},\,
\end{align}
Theorem 2.4 of~\cite{kovachki2020conditional} shows that 
$T^1$ transports $\rho_1$ to $\mu_1$, then 
$T^{2}$ transports $\rho_2$ to $\mu_{2|1}$.  
In particular, for $\rho_1$-a.e.\thinspace $x_1$, 
\begin{align} \label{eq:conditional_samplingT2}
    T^2(X_2;x_1) \sim \mu_{2|1}(\cdot|T^1(x_1))\,, \quad X_2 \sim \rho_2\,.
\end{align}
Thus, maps of the form in~\eqref{eq:cond_brenier} can perform conditional simulation. Several methods successfully learn these maps given only information from the joint distribution $\mu$. These include conditional normalizing flows and  
diffusion models; see Section~\ref{sec:related_work} for some examples. 
However, most existing approaches (i) do not provide explicit guarantees 
in terms of the number of samples required to obtain a good estimator,  
and (ii) are based on parametric models, the successful use of which requires tuning an unreasonable number of parameters. The latter is especially costly when using neural networks as the proposed estimator. 

\paragraph{Conditional Brenier maps.} 
Among all maps of the form in~\eqref{eq:cond_brenier} that perform conditional simulation, we seek
the unique transport whose component maps $T^1, T^2$ minimize the squared Euclidean displacement cost. We call such a transport a \emph{conditional Brenier map}, denoted $T_{\text{CB}}$. 
A theoretical approximation scheme for finding  $T_{\text{CB}}$ was proposed by \cite{carlier2010knothe}: For $t \in (0,1)$, define the positive definite matrix\looseness-1
\begin{align}\label{eq:At}
    A_t \defeq \text{diag}(\bm{1}_{d_1}, \sqrt{t}\bm{1}_{d_2}), 
\end{align}
where $\bm{1}_{d_1} \defeq (1,\ldots,1) \in \R^{d_1}$, with $d = d_1 + d_2$, and consider the weighted Euclidean cost function
\begin{align}\label{eq:ct_cost}
    c_t(x,y) \defeq \tfrac12\|A_t(x-y)\|^2_2\,.
\end{align}
\cite{carlier2010knothe} show\footnote{Their original result shows convergence of $T_t$ to the Knothe-Rosenblatt rearrangement, a map whose Jacobian is \emph{strictly-triangular}, by considering the weighted Euclidean cost $c_t(x,y) = \sum_{k=1}^d \frac{1}{2} t^{k-1} |x_k - y_k|^2$. 
The proof shows that each component of $T_t$ converges to an optimal map between one-dimensional conditional distributions. The same argument applies block-wise to $x_1,x_2$ to yield the map of the form in~\eqref{eq:conditional_samplingT2} whose Jacobian is \emph{block-triangular}.} the following:
\begin{theorem}[Convergence to conditional Brenier maps]\label{thm: cbconv_main}
Let $\rho,\mu\in\cP_2(\Omega)$ with $\rho$ having a density. Let $T_t$ denote the corresponding optimal transport map for the cost $c_t$ satisfying $(T_t)_\sharp \rho = \mu$.  
Then as $t\to 0$,
\begin{align*}
    \Delta(T_t, T_{{\rm{CB}}}) \defeq \| T_t - T_{{\rm{CB}}}\|^2_{L^2(\rho)} \to 0\,.
\end{align*}
\end{theorem}
\Cref{thm: cbconv_main} states that, in order to approximate conditional Brenier maps, it suffices to learn optimal transport maps pertaining to a rescaled quadratic cost function, namely $c_t$ in~\eqref{eq:ct_cost}. While this result is non-quantitative, it provides an optimal transport framework for conditional simulation.

\subsection*{Main contributions}
We propose an entropic estimator for conditional Brenier maps  based on the works of~\cite{pooladian2021entropic} and~\cite{carlier2010knothe}.
In fact, we propose a general framework that allows one to use \emph{any} estimator of the (optimal) transport map between two measures.

In addition, the risk incurred by any finite-sample estimator we propose, written $\widehat{T}_t$, has the following decomposition
\begin{align}\label{eq:estimator_decomp_intro}
\begin{split}
    \E\Delta(\widehat{T}_t,T_{\text{CB}}) &\lesssim \E\Delta(\widehat{T}_t,T_t) + \Delta(T_t, T_{\text{CB}}).  
\end{split}
\end{align}
We take steps toward quantifying~\eqref{eq:estimator_decomp_intro} 
by providing a full characterization of the two error terms in the Gaussian setting.
For non-Gaussian measures, we numerically evaluate the performance of the conditional entropic Brenier map relative to two baseline approaches, one based on the nearest-neighbor estimator~\citep{manole2021plugin} and another based on neural networks~\citep{amos2022amortizing}. 
 We show that our estimator is more tractable than these other approaches and requires less tuning 
to maintain performance on standard conditional simulation tasks.\looseness-1

\subsubsection*{Notation}
For $\Omega \sse \Rd$, we write $\cP_2(\Omega)$ as the set of probability measures over $\Omega$ with finite second moment.
Throughout, we write $\rho \in \cP_2(\Omega)$ as the source measure which will always have a density with respect to Lebesgue measure over a (compact convex) set $\Omega \sse \R^d$. We denote the target measure by $\mu \in \cP_2(\Omega)$ for the target measure. 
For $\Omega \sse \R^{d_1 + d_2}$, we explicitly write the joint density as $\rho(x_1,x_2)$, and use $\rho_1(x_1) \in \cP(\R^{d_1})$ (resp. $\rho_2(x_2) \in \cP(\R^{d_2})$) to denote the first (resp. second) marginal distribution. 
For a fixed $x \in \R^{d_1}$, we write $\rho_{2|1}(\cdot|x)$ to be the conditional distribution; we use the same convention for $\mu = \mu(y_1,y_2)$. We denote the push-forward constraint for a transport map $T$ satisfying $T(X) \sim \mu$ for $X \sim \rho$ as $T_\sharp \rho = \mu$. For a positive definite matrix $C \in \R^{d_1 + d_2}$, we write $C^{1/2}$ as the symmetric square-root of $C$ and $\mathsf{L}_C$ as the block-lower Cholesky decomposition of $C$: the 2 × 2-block
matrix (with blocks of size $d_1 \times d_1$ and $d_2 \times d_2$) which satisfies $\mathsf{L}_C\mathsf{L}_C^\top = C$.

\section{Background}\label{sec: background}

\subsection{Optimal transport}
We first recall some facts about optimal transport~\citep{Vil08,San15} for (weighted) squared-Euclidean cost functions of the form $c(x,y) = \tfrac12\|A(x-y)\|^2$ with $A \succ 0$.

We use three formulations of optimal transport (OT) between fixed measures $\rho,\mu \in \cP_2(\Omega)$ for such costs. First, the \emph{Monge} formulation is 
\begin{align}\label{eq:monge_c}
    \OT_0(\rho,\mu) = \inf_{T \in \cT(\rho,\mu)} \int \tfrac12\|A(x - T(x))\|^2 \dd \rho(x)  \,,
\end{align}
where $\cT(\rho,\mu)$ is a family of vector-valued \emph{transport maps} satisfying $X \sim \rho$, $T(X) \sim \mu$. For general $\rho$ and $\mu$, it is easy to see that such a minimizer need not exist, even for such nice cost functions. When they exist, we denote these minimizer by $T_0$, called a \emph{Brenier map}~\citep{Bre91}.

Next, the \emph{primal Kantorovich} problem~\citep{Kan42} is 
\begin{align}\label{eq:ot_primal}
     \OT_0(\rho,\mu) = \inf_{\pi \in \Pi(\rho,\mu)} \iint \tfrac12\|A(x-y)\|^2 \dd \pi(x,y)\,,
\end{align}
where $\Pi(\rho,\mu)$ is the set of couplings between $\rho$ and $\mu$, the set of joint measures with left-marginal $\rho$ and right-marginal $\mu$. Since $\rho$ and $\mu$ have finite second moment, \eqref{eq:ot_primal} is well-posed, and a minimizer, called the \emph{optimal plan}, always exists and is denoted by $\pi_0$.\looseness-1

Finally, the \emph{dual Kantorovich} optimization problem is
\begin{align}\label{eq:ot_dual}
     \OT_0(\rho,\mu) = \sup_{(f,g) \in L^1(\rho\otimes\mu)} \int f \dd \rho  + \int g \dd \mu\,,
\end{align}
which is subject to the constraint
\begin{align}\label{eq:constraint}
f(x) + g(y) \leq \tfrac12\|A(x-y)\|^2 \quad \forall (x,y) \in \Omega\times\Omega\,.
\end{align}
Under the same regularity conditions on the measures, a pair of maximizers $(f_0,g_0)$ exists, which we call optimal  \emph{Kantorovich potentials}. 

These three formulations are reconciled in the following theorem, which gives an explicit formula of the optimal transport map as a function of the optimal Kantorovich potential $f_0$, and says that $\pi_0$ is a deterministic function of $T_0$. The result for $A = I$ was proven by \citet{Bre91}, and generalized to a wide family of cost functions by \citet{gangbo1996geometry}. We present a simplified version of the statement here.
\begin{theorem}[Brenier's theorem for rescaled quadratic costs]
Suppose $\rho$ has a density with respect to the Lebesgue measure, and $\mu$ has finite second moment. Then, for costs $c(x,y) = \tfrac12\|x-y\|^2$, there exists a unique optimal transport map $T_0 \in \cT(\rho,\mu)$ that minimizes \eqref{eq:monge_c} given by 
\begin{align}\label{eq:otmap}
    T_0(x) = \nabla\phi_0(x)\,,
\end{align}
where $\phi_0 = \tfrac12\|\cdot\|^2 - f_0$ and $f_0$ is the optimal Kantorovich potential. Moreover, for costs $c(x,y) = \tfrac12\|A(x-y)\|^2$ with positive definite $A$, the optimal transport map between $\rho$ and $\mu$ is given by
\begin{align}\label{eq:otmap_A}
    T_A(x) = A^{-1} \tilde{T}_0(Ax)\,,
\end{align}
where $\tilde{T}_0$ is the Brenier map between the transformed measures $\rho_A \defeq (A\cdot)_\sharp\rho$  and $\mu_A \defeq (A\cdot)_\sharp\mu$.
\end{theorem}

\subsection{Entropic optimal transport}
Entropic regularization was initially introduced to improve computational tractability of the matching problem; 
see \citet{cuturi2013sinkhorn,PeyCut19} for computational insights, and \cite{genevay2019entropy} for a general overview 
in machine learning.

For $\eps > 0$, the primal entropic optimal transport problem amounts to adding the KL divergence as a regularizer to the primal Kantorovich problem:
\begin{align}\label{eq:eot_primal}
\begin{split}
    \OT_\eps(\rho,\mu) &\defeq \inf_{\pi \in \Pi(\rho,\mu)} \iint \tfrac12\|A(x-y)\|^2 \dd \pi(x,y) + \eps \kl{\pi}{\rho\otimes\mu}\,,
\end{split}
\end{align}
where we recall that when $\pi$ has a density with respect to $\rho\otimes\mu$, we have 
\begin{align*} 
    \kl{\pi}{\rho\otimes\mu} = \int \log \left(\frac{\dd \pi}{\dd (\rho\otimes\mu)} \right) \dd \pi\,,
\end{align*}
and otherwise has value ${+\infty}$. Due to the regularizer, \eqref{eq:eot_primal} is a strictly convex problem and admits a unique minimizer, called the \emph{optimal entropic plan}, written $\pi_\eps$.

Analogously, we have a dual formulation, written 
\begin{align}\label{eq:oteps_dual}
\begin{split}
     \OT_\eps(\rho,\mu) = \sup_{(f,g) \in L^1(\rho\otimes \mu)} &\int f \dd \rho + \int g \dd \mu + \eps \\
     &\qquad - \eps \iint e^{(f(x)+g(y)-\tfrac12\|A(x-y)\|^2)/\eps} \dd \rho(x)\dd\mu(y)\,. 
\end{split}
\end{align}
As $\eps \to 0$, the third term in~\eqref{eq:oteps_dual} converges to the hard-constraint in~\eqref{eq:constraint}. The maximizers of~\eqref{eq:oteps_dual} are called \emph{optimal entropic Kantorovich potentials},  written $(f_\eps, g_\eps)$.\looseness-1 

In addition, the entropic optimal transport problem exhibits the following \emph{primal-dual} recovery formula \citep{Csi75}, in which the dual variables give an explicit form for the primal solution via
\begin{align}\label{eq:primal_dual}
    \dd\pi_\eps(x,y) = e^{(f_\eps(x) + g_\eps(y)  - \tfrac12\|A(x-y)\|^2)/\eps}\dd\rho(x)\dd\mu(y).
\end{align}
We can extend the potentials $(f_\eps,g_\eps)$ to lie outside the support of $\rho$ and $\mu$ respectively, by appealing to the marginal constraints (see \cite{mena2019statistical,nutz2021entropic}). Henceforth, we write
\begin{align*}
\begin{split}
& f_\eps(x)\! = \!-\eps\log\int e^{(g_\eps(y) - \tfrac12\|A(x-y)\|^2)/\eps}\dd\mu(y) \quad (x \in \R^d) \\
    & g_\eps(y) \! = \! -\eps\log\int e^{(f_\eps(x) - \tfrac12\|A(x-y)\|^2)/\eps}\dd\rho(x) \quad (y \in \R^d)\,.
\end{split}
\end{align*}

\subsubsection{Entropic analogues to Brenier's theorem}\label{sec:entmap}
As before, consider the special case $c(x,y) = \tfrac12\|x-y\|^2$. We can write $\phi_\eps = \tfrac12\|\cdot\|^2 - f_\eps$ and $\psi_\eps = \tfrac12\|\cdot\|^2 - g_\eps$. Appealing to the expression given by the extended potentials, we have\looseness-1
\begin{align}\label{eq:phi_eps}
    \phi_\eps(x) = \eps \log \int e^{(x^\top y - \psi_\eps(y))/\eps}\dd\mu(y)\,.
\end{align}
Then, the \emph{entropic Brenier map} is given by 
\begin{align}\label{eq:entropic_brenier_map}
    T_\eps(x) \defeq \nabla \phi_\eps(x) = \E_{\pi_\eps}[Y|X=x]\,,
\end{align}
where the second equality follows from taking the gradient of \eqref{eq:phi_eps}.\footnote{This is permitted by an application of the dominated convergence theorem.} This explicit characterization of the entropic map (relating the gradient field to the barycentric projection of $\pi_\eps$) was introduced as a means of providing a computationally tractable estimator to the Brenier map \citep{pooladian2021entropic}. 

For costs of the form $\tfrac12\|A(x-y)\|^2$, we arrive at a similar formula to~\eqref{eq:otmap_A} as in the unregularized case (cf. \citet[Lemma 2]{klein2023learning})
\begin{align}\label{eq:entropic_brenier_rescaled}
    T_{A,\eps}(x) = A^{-1} \tilde{T}_\eps(Ax)\,,
\end{align}
where $\tilde{T}_\eps$ is the entropic Brenier map between the transformed measures $\rho_A \defeq (A\cdot)_\sharp\rho$  and $\mu_A \defeq (A\cdot)_\sharp\mu$.

\subsubsection{Computational considerations with and without regularization}\label{sec:computation}
Given i.i.d. samples $X_1,\ldots,X_n \sim \rho$ and $Y_1,\ldots,Y_n \sim \mu$, the primal Kantorovich problem can be computed by solving the linear program 
\begin{align}
    \widehat{P} \defeq \argmin_{P \in \mathsf{DS}_n} \langle C, P \rangle\,,
\end{align}
where $\mathsf{DS}_n$ is the set of doubly-stochastic $n \times n$ matrices (i.e., matrices with non-negative entries where the row and columns each sum to one), and $C_{ij} \defeq \tfrac12\|A(X_i - Y_j)\|^2$. The runtime for solving this linear program is well-known to be $\cO(n^3\log(n))$, with the caveat that the cost matrix $C \in \R^{n\times n}$ must be stored in memory; see \citet[Chapter 3]{PeyCut19}.\looseness-1

Incorporating entropic regularization has the benefit of improved runtime complexity. Given $n$ samples as before, 
we now solve the strongly convex program
\begin{align}\label{eq:eot_discret}
    \widehat{P}_{\eps} \defeq \argmin_{P \in \mathsf{DS}_n} \langle C, P \rangle + \eps H(P)\,,
\end{align}
where $H(P)$ is the discrete entropy of the matrix $P$. The most well-known algorithm for computing~\eqref{eq:eot_discret} is \emph{Sinkhorn's algorithm} \citep{sinkhorn1964relationship,cuturi2013sinkhorn}. 
Solving for $\widehat{P}_\eps$ entails an iterative approach, where we solve for discrete analogues of the optimal dual variables $(\hat{f}_\eps,\hat{g}_\eps) \in \R^n \times \R^n$. By way of~\eqref{eq:primal_dual}, this gives 
\begin{align}\label{eq:primal_dual_discrete}
    (\widehat{P}_\eps)_{ij} = e^{(\hat{f}_\eps)_i/\eps} e^{C_{ij}/\eps}e^{(\hat{g}_\eps)_j/\eps}\,.
\end{align}
In this setting, is it known that for a given $\eps > 0$, the matrix $\widehat{P}_\eps$ can be computed in $\tilde{\cO}(n^2/(\eps\delta))$ time, where $\delta>0$ is the desired tolerance to optimality \citep{AltWeeRig17}. For $n$ sufficiently large, the quadratic runtime complexity is a significant improvement over the cubic complexity of the unregularized solver. 

While the above approach is popular, one can also solve for the dual optimal vectors $(\hat{f}_\eps,\hat{g}_\eps) \in \R^n \times \R^n$ without storing the cost matrix $C \in \R^{n\times n}$. This results in a computationally feasible method for $n \gtrsim 10^5$. We refer to \citet[Chapter 4]{PeyCut19} for more details.

\section{A family of estimators for\\ conditional simulation}
\label{sec:gen_recipe}
Given $X_1,\ldots,X_n \sim\rho$ and $Y_1,\ldots,Y_n\sim\mu$, our three-step recipe to estimate conditional Brenier maps  $T_{\text{CB}}$ is summarized as follows: For fixed $t > 0$, let $A_t$ be given by \eqref{eq:At}. Then\looseness-15
\begin{enumerate}
    \item scale the data to obtain $\XX_t \defeq (A_tX_1,\ldots,A_tX_n)$ and $\YY_t \defeq (A_tY_1,\ldots,A_tY_n)$,
    \item learn an estimator $\widehat{T}_0$ of the Brenier map from the data $\XX_t$ to $\YY_t$,
    \item for a new sample $x \sim \rho$, unscale the pushforward map, resulting in 
\begin{align}\label{eq:hat_t_t}
    \widehat{T}_t(x) \defeq A_t^{-1} \widehat{T}_0(A_tx)\,.
\end{align}
\end{enumerate}
This general approach allows practitioners to consider any estimator for the optimal transport map $\widehat{T}_0$ on the basis of data. 
We now discuss some estimators $\widehat{T}_t$ in detail, and comment on their scalability and practicality.  
\subsection{Approach 1: Conditional entropic Brenier maps}
Our main estimator is based on the entropic Brenier map introduced by ~\cite{pooladian2021entropic}.
Let $(\hat{f}_{\eps,t},\hat{g}_{\eps,t})$ be the output of Sinkhorn's algorithm on the scaled data $\XX_t$ and $\YY_t$. Following~\eqref{eq:entropic_brenier_map} and \eqref{eq:entropic_brenier_rescaled}, the estimator is 
\begin{align}
    \widehat{T}_{\eps,t}(x) 
    =   \sum_{i=1}^n Y_i \frac{e^{((\hat{g}_{\eps,t})_i - \tfrac12\|A_t(x - Y_i)\|^2)/\eps}}{\sum_{j=1}^n e^{((\hat{g}_{\eps,t})_j - \tfrac12\|A_t(x - Y_j)\|^2)/\eps} }\,.
\end{align}
This non-parametric estimator can be evaluated in $\cO(n)$ time, and estimated in $\cO(n^2/\eps)$ runtime. More importantly, as Sinkhorn's algorithm is GPU-friendly, this estimator is scalable to over $n \gtrsim 10^5$ sample points.
\subsection{Approach 2: Nearest-Neighbor estimator}
Another non-parametric estimator is due to~\cite{manole2021plugin} which was adopted by~\cite{hosseini2023conditional} for the purpose of conditional simulation. Let $\widehat{P}_t$ be the discrete optimal transport matching computed on the scaled data $\XX_t$ and $\YY_t$. For $x \in \Rd$, this estimator is\looseness-1
\begin{align}\label{eq:Tnn}
    \widehat{T}_{\text{NN},t}(x)  
    =  \sum_{i=1}^n \bm{1}_{V_{i}}(A_tx) Y_{\widehat{P}_t(i)}\,,
\end{align}
where $(V_i)_{i=1}^n$ are Voronoi regions
\begin{align*}
    V_i \defeq \{ x \in \Rd \colon \ \|x - A_tX_i\| \leq \| x - A_tX_k\|\, \forall \ k \neq i\}\,.
\end{align*}
Computing the closest $X_i$ to a new sample $x$ has runtime $\cO(n\log(n))$, though the overall runtime of the estimator is still $\cO(n^3)$, since $\widehat{P}_t$ needs to be initially computed (recall the discussions in~\cref{sec:computation}). An important caveat to this estimator is that it is only defined on the \emph{in-sample} target points. 
\subsection{Approach 3: Neural networks}
Neural networks are widely used to estimate optimal transport maps on the basis of data, and fall under our proposed framework as well. We follow the work of~\cite{amos2022amortizing}, which trains two multilayer perceptrons (MLPs) neural networks $\phi_\theta$ and $x_\vartheta$ to define the map. In brief, the network $x_\vartheta :\R^d \to \R^d$ models the reverse transport map and $\phi_\theta :\R^d \to \R$ models the Brenier potential. We can fit these neural networks using stochastic gradient descent on minibatches of our fixed data $\XX_t$ and $\YY_t$, and express our estimator as
\begin{align}
    \widehat{T}_{{\mathrm{MLP}},t}(x) = A_t^{-1}\nabla\widehat\phi_\theta(A_t x)\,.
\end{align}
We provide more details on the training algorithm in ~\cref{app:NNtraining}, though we refer the interested reader to \cite{amos2022amortizing} for precise details.

\section{Theory for conditional simulation:
The Gaussian case}\label{sec:gaussian_theory_results}
Our goal is to estimate the conditional Brenier map $T_{\rm{CB}}$ on the basis of samples. To explicitly quantify the error, we recall that the $L^2(\rho)$ risk incurred by all estimators from the preceding section can be decomposed as
\begin{align}\label{eq:estimator_decomp}
\begin{split}
    \E\Delta(\widehat{T}_t,\TCB) &\lesssim \E\Delta(\widehat{T}_t,T_t) + \Delta(T_t, \TCB) =:  (\mathsf{T1}) + (\mathsf{T2})\,,
\end{split}
\end{align}
where we recall $\Delta(f,g)\defeq \|f-g\|^2_{L^2(\rho)}$. $(\mathsf{T1})$ denotes the statistical error, which depends on the method of choice, whereas $(\mathsf{T2})$ denotes the approximation error of 
$T_t$ to $T_{\rm{CB}}$. Ideally, our choice of $t$ should decrease with the number of samples, resulting in a consistent estimator.

In this section, we make the estimates for~\eqref{eq:estimator_decomp} quantitative 
by considering a Gaussian-to-Gaussian transport problem, where we can leverage closed-form expressions for the estimators as a first step to understanding the error. All proofs in this section are deferred to~\Cref{app:proofs}. We will require the following assumption:
\begin{description}
\item \textbf{(G)} Let $\rho=\cN(0,I_d)$ and $\mu = \cN(m,\Sigma)$ for $\Sigma \succ 0$. 
\end{description}
We first collect three closed-form expressions of interest.
\begin{prop}[Closed-form expressions]\label{prop:closedform}
Suppose $\rho$ and $\mu$ satisfy \textbf{(G)}. Let $T_{\text{B}}$ be the optimal transport map under the squared-Euclidean cost, $T_t$ be the optimal transport map for the $c_t$ cost, and $T_{\text{CB}}$ be the conditional Brenier map, all between $\rho$ and $\mu$. Then,\looseness-1
\begin{align*}
    T_{\text{B}}(x) &= m + \Sigma^{1/2}x\,, \\
    T_{t}(x) &= m + A_t^{-2}(A_t^{2}\Sigma A_t^{2})^{1/2}x\,, \\
    T_{{\rm{CB}}}(x) &= m + \mathsf{L}_\Sigma x\,.
\end{align*}
\end{prop}

Our main result of this section is the following theorem. The proof follows immediately from balancing the results of  \cref{prop:stats_gaussian_t} (see \cref{sec:conv_t1}) and \cref{prop:quantitative_gaussian_result} (see \cref{sec:conv_t2}), and is therefore omitted.
\begin{theorem}
Suppose $\rho$ and $\mu$ satisfy \textbf{(G)}, and consider the following plug-in estimator
\begin{align}\label{eq:plugin_gaussian}
    \widehat{T}_t(x) = \widehat{m} +  A_t^{-2}(A_t^2\widehat{\Sigma}A_t^2)^{1/2}x\,,
\end{align}
where $\widehat{m}$ and $\widehat{\Sigma}$ are the empirical mean and covariance derived from the i.i.d.\thinspace samples from $\mu$. 
If $t(n) \asymp n^{-1/3}$, and $n$ is sufficiently large, the plug-in estimator $\widehat{T}_t$ achieves the estimation rate
\begin{align}
    \E \|\widehat{T}_t - T_{{\rm{CB}}}\|^2_{L^2(\rho)} \lesssim n^{-2/3}\,.
\end{align}
\end{theorem}
\begin{remark}
    It is worth mentioning that this result is far from tight.
     A plug-in estimator of $\mathsf{L}_\Sigma$ would directly provide parametric rates of estimation. However, this estimation approach is outside the spirit of our work, as it avoids the rescaled quadratic cost entirely.
\end{remark}

\subsection{Convergence of $(\mathsf{T1})$} \label{sec:conv_t1}
Our first step is to understand how the scaling of the cost impacts the overall statistical performance of our estimator. We ultimately expect $t = t(n) \searrow 0$, which will (negatively) impact the convergence rate. Indeed, if $t$ were ultimately fixed and not decaying with $n$, then the rescaling impacts the rate of estimation by at most a universal, albeit large, constant, but we would never be able to converge to the conditional Brenier map.

\begin{figure}[t]
    \centering
    \includegraphics[width=0.55\textwidth]{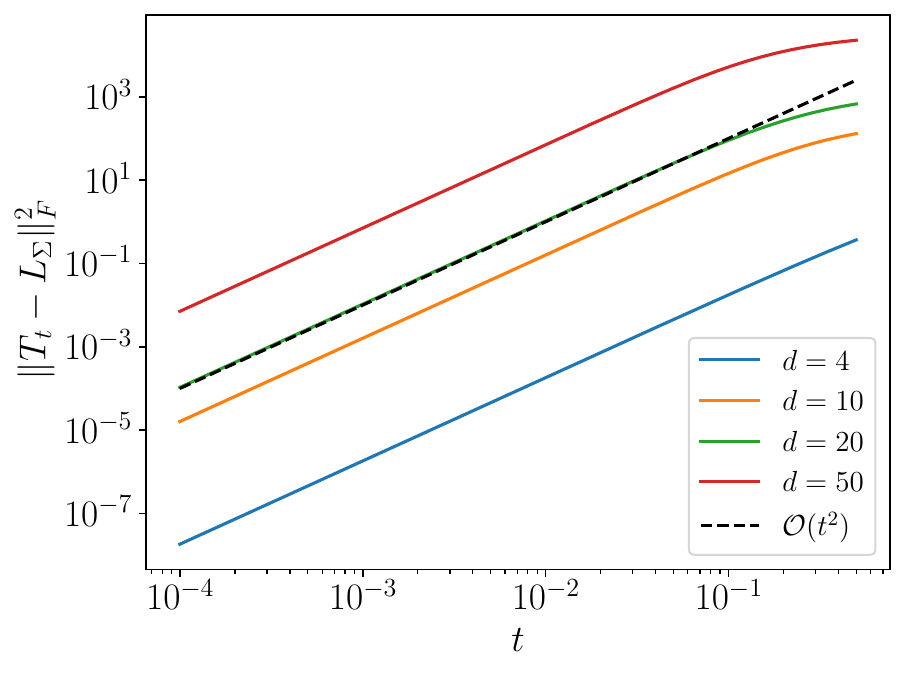}
    \caption{We observe that $(\mathsf{T2})$ asymptotically converges with a rate of $\mathcal{O}(t^2)$ convergence rate with randomly generated covariance matrices of block-type.}\label{fig:convrate_gaussians}
\end{figure}

The following proposition tells us how the scaling of the cost impacts the statistical rates of convergence.
\begin{prop}\label{prop:stats_gaussian_t}
Suppose $\rho$ and $\mu$ satisfy \textbf{(G)}, and consider the plug-in estimator from \eqref{eq:plugin_gaussian}.
The statistical rate of convergence is 
\begin{align}
    (\mathsf{T}1) \defeq \E\|\widehat{T}_t - T_t\|^2_{L^2(\rho)}\lesssim t^{-1}n^{-1}\,.
\end{align}
\end{prop}

\subsection{Convergence of $({\mathsf{T2}})$}\label{sec:conv_t2}
We now quantify the approximation error $(\mathsf{T2})$ in~\eqref{eq:estimator_decomp}. The proof is based on a careful Taylor-expansion argument that is made tractable in the Gaussian-to-Gaussian setting.

\begin{prop}\label{prop:quantitative_gaussian_result}
Suppose $\rho$ and $\mu$ satisfy \textbf{(G)}. Then $T_t$ for $t$ sufficiently small converges to the conditional Brenier map at the rate 
\begin{align}\label{eq:quantitative_gaussian_result}
    (\mathsf{T2}) \defeq \|T_t - T_{{\rm{CB}}}\|^2_{L^2(\rho)} \lesssim_{\Sigma} t^2\,.
\end{align}
\end{prop}
\cref{fig:convrate_gaussians} supports \cref{prop:quantitative_gaussian_result}. Note that as $d$ increases, the constant in~\eqref{eq:quantitative_gaussian_result} increases, though the convergence \emph{rate} appears to always be $\cO(t^2)$ for $t$ sufficiently small. Details of the example are provided in~\cref{app:convrate_gaussians}.

\subsection{Impact of the entropic bias}
Unlike the Nearest-Neighbor or the MLP estimators, the entropic Brenier map comes with an explicit bias as an artifact of the $\eps$-regularization scheme. Here, we want to scale both parameters $\eps,t \to 0$ such that the following quantity converges
\begin{align} \label{eq:approx_entropic}
\begin{split}
    \|T_{\eps,t} - T_{\text{CB}}\|^2_{L^2(\rho)} &\lesssim \|T_{\eps,t} - T_{t}\|^2_{L^2(\rho)} + \|T_{t} - T_{\text{CB}}\|^2_{L^2(\rho)}\,.
\end{split}
\end{align}
Ultimately, we would like to provide a general rule for selecting the entropic parameter $\eps$ as a function of $t$. We expect $t = t(n)$ to decrease as the number of samples $n$ increases, and similarly for $\eps = \eps(n)$. The following result controls the entropic bias (first term in~\eqref{eq:approx_entropic}) for rescaled quadratic costs in the Gaussian setting. 
\begin{prop}\label{prop:entropic_bias}
Let $\rho=\cN(0,A)$ and $\mu=\cN(b,B)$, and let $c_S(x,y) \defeq \tfrac12\|S(x-y)\|^2$ for some positive-definite matrix $S$. Let $T_{\eps,S}$ (resp. $T_S$) be the entropic Brenier map (resp. Brenier map) between $\rho$ and $\mu$. It holds that for $\eps > 0$ 
\begin{align}
    \|T_{\eps,S} - T_S\|^2_{L^2(\rho)} \lesssim \tr(S^{-4}) \eps^2 + \cO_{A,B,S}(\eps^4)\,,
\end{align}

\end{prop}
As a special case, we can take $S = A_t$ and, combined with \cref{prop:quantitative_gaussian_result}, we obtain the following result for the risk of the finite-sample entropic estimator. 
\begin{theorem}\label{thm:entropic_gaussian}
Let $T_{\eps,t}$ be the entropic Brenier map under the $c_t$ cost between $\rho$ and $\mu$ satisfying \textbf{(G)}, and let $T_{\text{CB}}$ be the conditional Brenier map. For $\eps \asymp t^{2}$,
\begin{align*}
    \|T_{\eps,t} - T_{{\rm{CB}}}\|^2_{L^2(\rho)} \lesssim_{\Sigma,d} t^2\,. 
\end{align*}
\end{theorem}

\section{Numerical experiments}\label{sec:experiments_main}
We study our proposed map estimators on several experiments already present in the literature. Our experiments fall into two broad categories: quantitative and qualitative comparisons. In the former, we consider settings where the true conditional Brenier map is known so that sampling the target conditioning or computing the mean-squared error (MSE) from \eqref{eq:estimator_decomp} is possible. For the latter, such a map is not known (which is the case in practice), and so we can only visually compare the generated and target conditional samples. 

Recall that two of our estimators are non-parametric: the entropic Brenier map (EOT) and the Nearest-Neighbor map (NN), while the neural network approach (MLP) is parametric. The entropic Brenier map takes the best aspects of both other estimators: it is scalable to many samples via Sinkhorn's algorithm and only requires tuning one extra variable, $\eps$, as opposed to changing minibatch size, learning schedules, architecture, etc., which is required with neural networks.

In all of our experiments we take the reference measure to be of the form $\rho = \mu_1 \otimes \rho_2$ so that the first component of the conditional Brenier map is $T^1(x_1) = \text{Id}(x_1)$. We can then sample the conditional distribution $\mu_{2|1}$ using $T^2$ alone, i.e., \eqref{eq:conditional_samplingT2} becomes $T_2(X_2;x_1) \sim \mu_{2|1}(\cdot|x_1)$ for $X_2 \sim \rho_2$ and any $\rho_1$-a.e.\thinspace $x_1$.

Our code is publicly available at \href{https://github.com/aistats2025condsim/ConditionalBrenier}{https://github.com/aistats2025condsim/ConditionalBrenier}; details of the experiments not mentioned in the main text are deferred to~\cref{app:experiments}. 

\subsection{Quantitative comparisons}\label{sec:experiments_quant}
\subsubsection{Non-linear conditional Brenier maps} \label{sec:tanh_experiment}

Here we consider a set of two-dimensional joint distributions $\mu(x_1,x_2)$ where $x_1 \sim \mu_1 = U[-3,3]$ and the conditionals $\mu_{2|1}(\cdot|x_1)$ are sampled as follows:
\begin{itemize} \itemsep-3pt
\item[] \texttt{tanhv1}: $X_2 = \tanh(x_1) + \xi \quad \xi \sim \Gamma(1,0.3)$
\item[] \texttt{tanhv2}: $X_2 = \tanh(x_1 + \xi), \quad \xi \sim \mathcal{N}(0,0.05)$
\item[] \texttt{tanhv3}: $X_2 = \xi\tanh(x_1), \quad \xi \sim \Gamma(1,0.3).$
\end{itemize}
For each case, we generate $n = 5000$ independent sets of samples from the source $\rho = \mu_1 \otimes \mathcal{N}(0,1)$ and target distribution $\mu$. For each batch of samples, we compute the three estimators and generate $2000$ samples from the approximate conditionals using the estimated map.

To evaluate the error in the conditional distributions for each value of $x_1$, we 
compute a Monte-Carlo estimate of the following expected error between conditionals
\begin{align} \label{eq:ExpectedConditionalError}
    \E_{X_1 \sim \mu_1}[\mathsf{D}(\mu_{2|1}(\cdot|X_1), \widehat{T}^2(\cdot;X_1)_\sharp\rho_2)]\,,
\end{align}
where $\mathsf{D}$ is some discrepancy over probability measures. 

To populate \cref{tab:tanh}, we used the same $t \in (0,1)$ for all methods to define the rescaled data. For the approach based on entropic OT, we chose $\eps = t/5$. We randomly sample $50$ i.i.d.\thinspace conditional variables from $\mu_1$ and evaluate the error in~\eqref{eq:ExpectedConditionalError} with $\mathsf{D}$ taken to be $W_2$ and the Maximum-Mean-Discrepancy (MMD) metric. We repeat each experiment 10 times and report the standard deviations of the errors. For all cases, we observe that the EOT and NN estimators yield the smallest error between the true and approximate conditionals.  
\begin{table}[!ht]
\centering
\resizebox{0.7\textwidth}{!}{
\begin{tabular}{cccc}
{Dataset} & {Method} & {$W_2$ Error ($10^{-2}$)} & {MMD Error ($10^{-2}$)} \\ \hline
\multirow{3}{*}{\centering \texttt{tanhv1}} & EOT & 4.26 $\pm$ 0.62 & 1.31 $\pm$ 0.51 \\ 
                                        & NN  & 5.05 $\pm$ 0.43 & 1.34 $\pm$ 0.43 \\ 
                                        & MLP & 5.98 $\pm$ 0.88 & 1.97 $\pm$ 0.67 \\ \hline
\multirow{3}{*}{\centering \texttt{tanhv2}} 
& EOT & 3.15 $\pm$ 0.81 & 2.46 $\pm$ 0.86 \\ 
& NN  & 3.83 $\pm$ 0.68 & 1.91 $\pm$ 0.58 \\ 
& MLP & 14.39 $\pm$ 3.86 & 15.50 $\pm$ 5.26 \\ \hline
\multirow{3}{*}{\centering \texttt{tanhv3}} 
& EOT & 2.12 $\pm$ 0.67 & 18.43 $\pm$ 7.78 \\ 
& NN  & 1.77 $\pm$ 0.73 & 10.39 $\pm$ 7.43 \\ 
& MLP & 5.42 $\pm$ 0.61 & 27.19 $\pm$ 2.21 \\ \hline
\end{tabular}
}
\caption{Expected error between conditional distributions for three test problems; we chose $t=6\cdot10^{-2}$, which seemed to provide the best performance for all methods on average.}
\label{tab:tanh}
\end{table}
 
\subsubsection{Gaussian setting} \label{sec:gaussian_experiment}
Next, we examine the MSE of the estimated map on a $d=4$-dimensional Gaussian source and target measure with $d_1 = 2$ and $d_2 = 2$. Figure~\ref{fig:Gaussian_conv} compares the error to the true map with increasing sample size for $t \asymp n^{-1/3}$ and $\eps \asymp t^2$ (following the theoretical analysis in Section~\ref{sec:gaussian_theory_results}). While we do not believe these rates are optimal, this is a first step toward demonstrating consistency of our proposed estimator. 
\begin{figure}[!ht]
\centering
\includegraphics[width=0.6\textwidth]{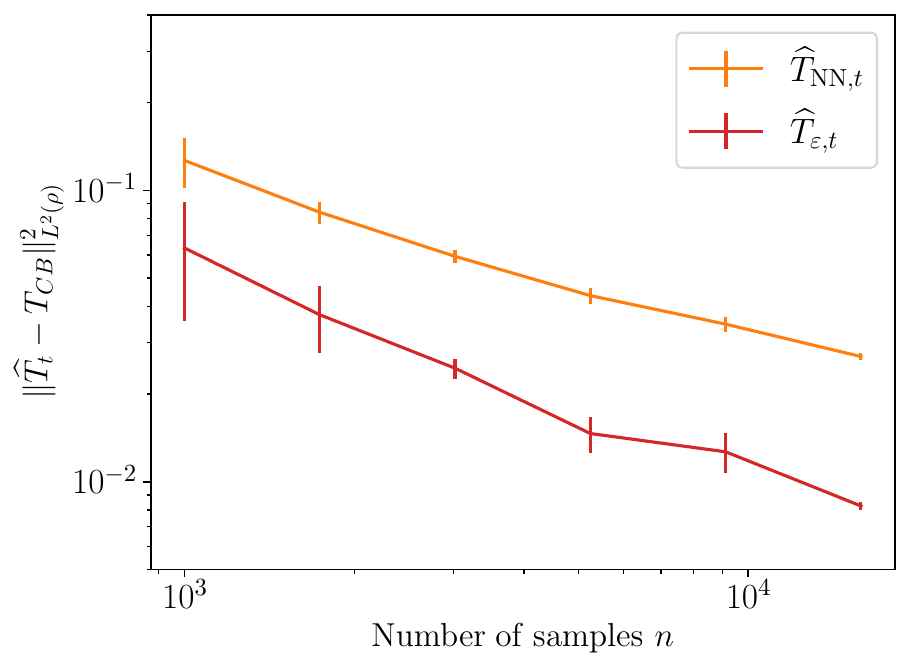}
\caption{MSE of the estimated map for the NN and EOT estimators for a multivariate Gaussian problem with increasing sample size $n$.\label{fig:Gaussian_conv}}
\end{figure}

\subsection{Qualitative comparisons}\label{sec:experiments_qual}

\subsubsection{Two-dimensional distribution} \label{sec:banana_experiment}

Here we visualize the approximated conditional distributions $\mu_{2|1}(\cdot|x_1)$ for different estimators and conditioning variables $x_1$. We consider a two-dimensional distribution $\mu$ where the data is sampled as follows: $X_2 \sim \mathcal{N}(0,1)$ and $X_1 = x_2^2 - 1 + \xi$ with $\xi \sim \mathcal{N}(0,1)$. 

Figure~\ref{fig:banana_distributions} shows the target samples from $\mu$ for learning the maps and the generated samples from the estimated maps $\widehat{T}^2(\cdot,x_1)_\sharp \rho_2$ for $x_1 \in \{-0.5,3\}$. We observe that the maps approximate distributions with both unimodal ($x_1=-0.5$) and bimodal structure ($x_1=3$). The closest is agreement found using the entropic map. 

\begin{figure}[!ht]
\centering
\includegraphics[width=0.27\textwidth]{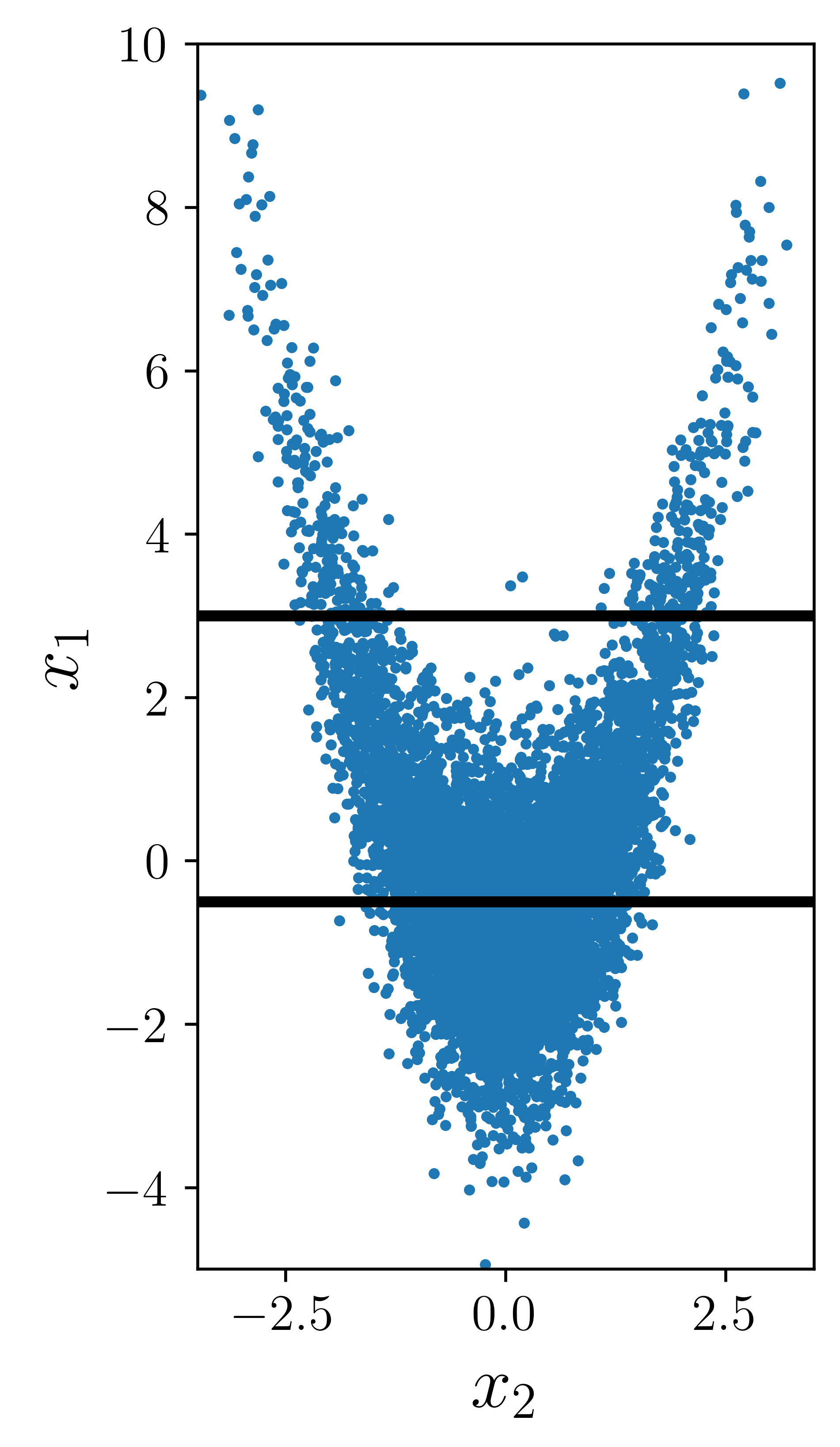}
\includegraphics[width=0.35\textwidth]{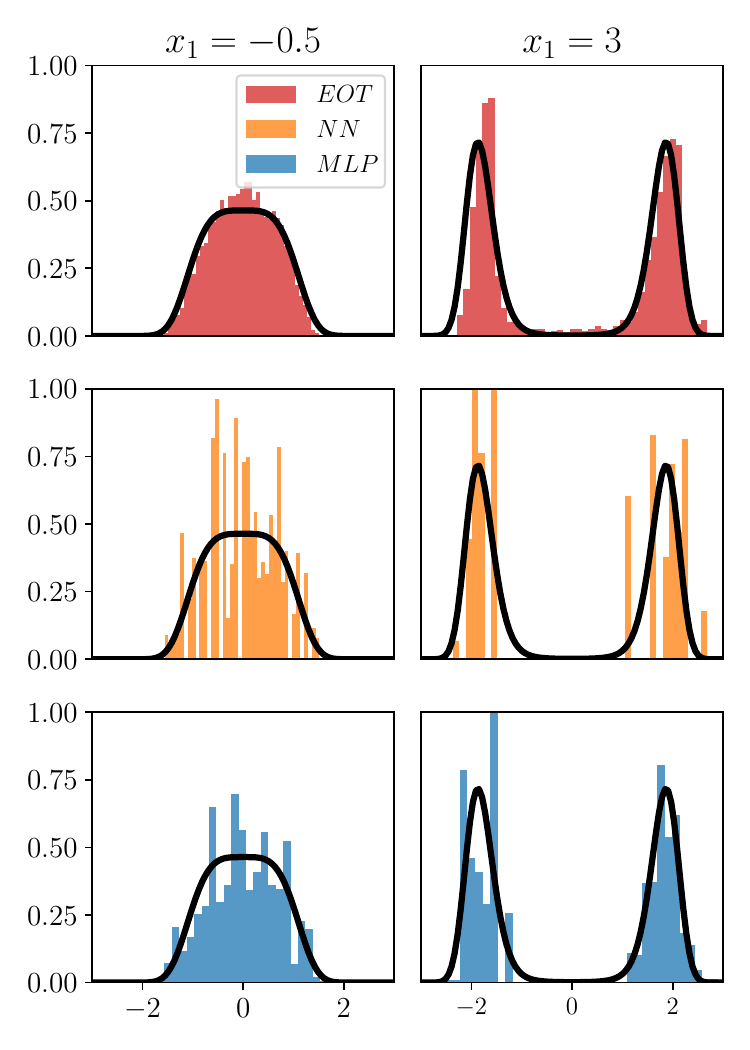}
\caption{Left: Joint samples of $\mu(x_1,x_2)$ with slices at the conditioning values of interest $x_1 \in \{-0.5,3\}$. Right: Generated samples from the EOT, NN, and MLP maps with the true density $\mu_{2|1}(\cdot|x_1)$ in black. \label{fig:banana_distributions}}
\end{figure}

\subsubsection{Posterior of Lotka--Volterra model} \label{sec:Lotka-volterra}
Lastly, we apply the entropic estimator to  sample from the posterior distribution $\mu_{2|1}(\cdot|x_1)$ of a Bayesian inference problem for parameters $X_2 \in \R^4$ in a ordinary differential equation that models population dynamics given one observation $x_1 \in \R^{18}$. Figure~\ref{fig:LVmodel} (left) presents $2.5 \times 10^4$ samples generated using the estimated entropic estimator $\widehat{T}_{\eps,t}$ for $\eps = t = 5\!\times\! 10^{-3}$ and $n = 10^5$. The other estimators (i.e., NN and MLP) are not scalable for this high-dimensional application. We compare the generated samples to samples obtained from an adaptive Metropolis Markov chain Monte Carlo sampler, where we observe close agreement of the concentration and correlations between the conditional distributions. Moreover, the true parameter (red) that generated the observation $x_1$ is covered by the predicted uncertainty.  

\begin{figure}
\centering 
\includegraphics[width=0.48\textwidth]{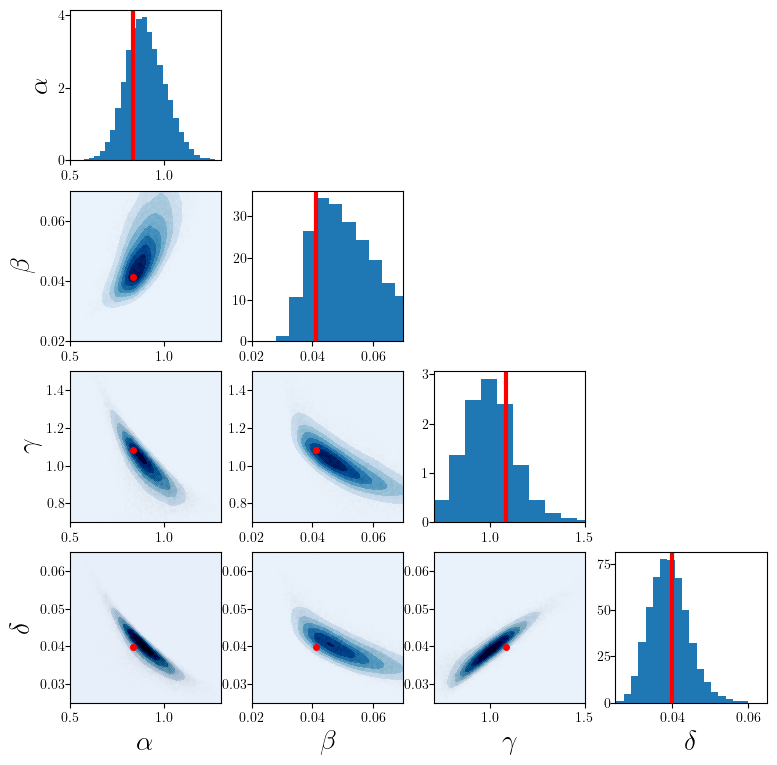}
\includegraphics[width=0.48\textwidth]{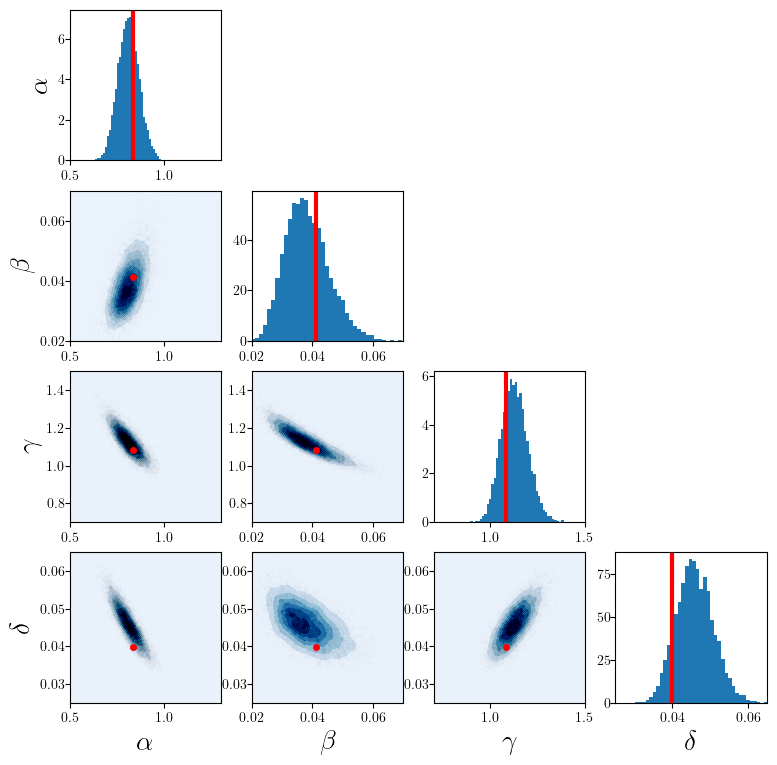}
\caption{Comparison of samples from the posterior distribution using the entropic estimator (left) and an adaptive MCMC sampler (right). \label{fig:LVmodel}}
\end{figure}

\section{Related work} \label{sec:related_work}
\paragraph*{Estimation of (entropic) Brenier maps.}
The statistical estimation of optimal transport maps has been studied by e.g,~\citet{deb2021rates,hutter2021minimax, manole2021plugin,divol2022optimal}. To the best of our knowledge, the estimation of optimal transport maps for other costs has not appeared in the literature so far. The works of~\cite{fan2021scalable,klein2023learning,pooladian2023neural} deviate from the squared-Euclidean cost, considering task-specific costs, such as Lagrangian costs to incorporate barriers in transport, or proximal costs for sparse displacements. Statistical analyses for entropic Brenier maps (again, for the squared-Euclidean cost) have themselves been studied in numerous works; see~\cite{pooladian2021entropic,goldfeld2022limit,goldfeld2022statistical,rigollet2022sample,pooladian2023minimax,stromme2023minimum}.\looseness-1

On the empirical side,  \cite{hosseini2023conditional} and \cite{alfonso2023generative} estimate flows or transport maps based on the notion of rescaled costs by leveraging \cite{carlier2010knothe}. Neither work quantitatively assesses the performance of their estimator as $n\to\infty$ or provide a (heuristic) rule for $t = t(n)\to 0$. In this sense, our work takes 
a first step to making their estimators statistically rigorous.\looseness-1

\paragraph*{Knothe--Rosenblatt map computation.} On the statistical side, our work is related to that of  of~\cite{irons2022triangular} and~\cite{wang2022minimax}. The authors of these works study the statistical estimation of the Knothe--Rosenblatt (KR) rearrangement (the strictly triangular map, rather than the block-triangular map considered here) via maximum likelihood estimation with the goal 
of obtaining rates of convergence in the KL divergence.  
Unlike these works, we are ultimately interested in procedures that explicitly recover the conditional Brenier map. Finally, adapted optimal transport~\cite{eckstein2024computational, backhoff2017causal,  gunasingam2024adapted} incorporates causal constraints to recover a transport plan (in the sense of KR) that is useful for conditional simulation, but developing tractable estimators remains challenging.\looseness-1

\paragraph*{Methods for conditional simulation}
Broadly, computational transport approaches for conditional simulation parameterize the transport map using either conditional normalizing flows~\cite{baptista2023representation, wang2023efficient} or generative adversarial networks~\cite{baptista2020conditional}. Alternatively, the transport can be defined from the flow map of a (possibly stochastic) differential equation as in~\cite{batzolis2021conditional, shi2022conditional, albergo2024stochastic}. Recently, a dynamic formulation of conditional optimal transport was proposed in~\cite{kerrigan2024dynamic}. These dynamic formulation use neural network-based methods, which require choosing many hyper-parameters in practice as compared to the proposed approach.

\section{Conclusion}
We put forth a statistically motivated approach for conditional simulation by estimating conditional Brenier maps. We propose a tractable estimator based on entropic optimal transport that approximates conditional Brenier maps from joint samples. Moreover we provide a near-complete statistical characterization of the estimator in a Gaussian-to-Gaussian setting. A natural and essential direction for future work is to extend the theoretical analyses beyond  Gaussians. Similarly, it would be interesting to extend our proposed estimator to the recent Schr\"odinger bridge estimator developed by \cite{pooladian2024plug} for the purposes of conditional simulation.

\section*{Acknowledgements}
AAP thanks NSF grant DMS-1922658 and Meta AI Research for financial support. JNW is supported by the Sloan Research Fellowship and NSF grant DMS-2339829.

\bibliography{references}

\appendix
\section{Experimental details} \label{app:experiments}

\subsection{Details of neural network estimator} \label{app:NNtraining}
For this estimator, we use a neural network to parameterize the Brenier potential $\phi_{\theta}$ and reverse transport map $x_{\vartheta}$. The neural networks have 4 hidden layers with 128 hidden units in each layer. We use an Adam optimizer to learn the parameters using a 256 batch size at each iteration and 5000 total iterations. The learning rate follows a cosine scheduler starting from the initial value of $10^{-2}$, which decays over $5000$ iterations by the multiplier factor of $10^{-3}$. The neural networks implementation and training is performed using \texttt{OTT-JAX}~\citep{cuturi2022optimal} using their ``Neural Dual" solver tutorial code. 

\subsection{Computing Figure~\ref{fig:convrate_gaussians}}\label{app:convrate_gaussians}

For each total dimension $d \in \{4,10,20,50\}$, we randomly sample a target covariance matrix $\Sigma \in \R^{d \times d}$ using the following procedure. We sample matrices $A \in \R^{d_1 \times d_1}, B \in \R^{d_2 \times d_1}$ and $C \in \R^{d_2 \times d_2}$ with independent standard Gaussian entries. We then construct the block covariance matrix $\Sigma$ of the form
$$\Sigma = \begin{bmatrix} AA^\top & AA^\top B^\top \\  BAA^\top & BAA^\top B^\top + 0.01 I_d \end{bmatrix}. $$
Given the target $\mu = \mathcal{N}(0,\Sigma)$ and standard Gaussian source $\rho = \mathcal{N}(0,I_d)$, we compute the lower block Cholesky factor $L_\Sigma \in \R^{d \times d}$ satisfying $\Sigma = L_\Sigma L_\Sigma^{\top}$ to define the conditional Brenier map $T_{\textrm{CB}}$ in~\cref{prop:closedform}. This map is compared to $T_t(x) = A_t^{-2}(A_t^2\Sigma A_t^2){1/2}x$ in~\cref{prop:closedform} to evaluate the squared errors in Figure~\ref{fig:convrate_gaussians}. Lastly, we note that for a standard Gaussian source we have 
\begin{align*}
    \|T_t - T_{\textrm{CB}}\|_{L^2(\rho)}^2 &= \mathbb{E}_{x \sim \rho} \text{Tr}\left(( A_t^{-2}(A_t^2\Sigma A_t^2)^{1/2} - L_\Sigma)xx^\top(A_t^{-2}(A_t^2\Sigma A_t^2)^{1/2} - L_\Sigma)^\top\right) \\
    &= \| A_t^{-2}(A_t^2\Sigma A_t^2)^{1/2} - L_\Sigma \|_F^2.
\end{align*}

\subsection{Details for Section~\ref{sec:experiments_quant}}\label{app:experiments_quant}

For both the quantitative comparisons, we construct the entropic estimator using the \texttt{Optimal Transport Tools} (\texttt{OTT-JAX}) package~\citep{cuturi2022optimal}. We set the maximum number of iterations for the Sinkhorn solver to $5000$ and use a tolerance of $10^{-3}$ for checking convergence. For the nearest-neighbour estimator, we use the \texttt{Scikit-Learn}~\citep{pedregosa2011scikit} package to compute the nearest neighbor points in the source dataset. 

For the \texttt{tanh} experiments in~\cref{sec:tanh_experiment} we set $t = 6\cdot 10^{-2}$ and $\eps = t/5$ for all examples. For the Gaussian experiment in~\cref{sec:gaussian_experiment}, we set $t = 0.1n^{-1/5}$ and $\eps = t^2$ for all considered sample sizes $n \in [10^{3}, 10^{4}]$. The errors for the estimated maps in Figure~\ref{fig:Gaussian_conv} are computed using Monte Carlo with $10^4$ independent samples from the source distribution.

\subsection{Details for  Section~\ref{sec:experiments_qual}}\label{app:experiments_qual}

For the experiment in~\cref{sec:banana_experiment}, we use $n=7500$ target samples from $\mu$ to build the estimators and set $t =6 \cdot 10^{-2}$. We use $\eps = t/5$ for the entropic estimator. The entropic and nearest-neighbor estimators are computed using \texttt{OTT} and \texttt{Scikit-Learn}, respectively, using the parameters listed in~\cref{app:experiments_quant}. To plot the histograms in Figure~\ref{fig:banana_distributions}, we generate 5000 samples from the estimated conditionals using the second component of the estimated maps $\widehat{T}^2(\cdot,x_1)$ for each $x_1$.  

For the experiment in~\cref{sec:Lotka-volterra}, we consider a target distribution specified by a prior $\mu_2(x_2)$ over parameters $x_2 = (\alpha,\beta,\delta,\gamma) \in \R^4$ that define the right-hand-side nonlinear ODE, and a likelihood function $\mu_{1|2}(x_1|x_2)$ given by the mapping from parameters to noisy observations of the ODE $x_1 \in \R^{18}$. Given the parameters, let $P(t) = (P_1(t),P_2(t)) \in \R_{+}^2$ describe the populations of predator and prey over time in an environment, which evolve according to the coupled ODES 
\begin{align*}
\frac{dP_1}{dt} &= \alpha P_1(t) - \beta P_1(t)P_2(t) \\
\frac{dP_2}{dt} &= - \gamma \alpha P_1(t) + \delta P_1(t)P_2(t),
\end{align*}
with the initial condition $P(0) = (30,1)$. To generate the observations, we simulate the ODEs for $t \in [0,20]$ and sample $(\log (X_1)_{2k-1},\log(X_1)_{2k}) \sim \mathcal{N}(\log P(k \Delta t_{obs}), \sigma^2 I_2)$ with $\Delta t_{obs} = 2$ and $\sigma = 0.01$ for $k = 1,\dots,9$. This sampling process defines the likelihood function $\mu_{1|2}(\cdot|x_2)$. We set the prior distribution to be log-normal, i.e., $\log X_2 \sim \mathcal{N}(m, 0.5I_4)$ with mean $m = (-0.125,-3,-0.125,-3)$. To generate the target dataset $\{(X_1^i,X_2^i)\}_{i=1}^n \sim \mu$, we sample a set of parameters $X_2^i \sim \mu_2(\cdot)$ and paired observations $X_1^i \sim \mu_{1|2}(\cdot|x_2 = X_2^i)$. 

To construct the entropic estimator for the conditional Brenier map, we consider the parameters $t = 0.01$, $\eps = 0.005$, and a tolerance for the Sinkhorn solver of $10^{-4}$. The estimator is computed using the \texttt{PyKeOps} package for scalability of kernel operations on the GPU with large sample sizes $n$~\citep{JMLR:v22:20-275}. For the comparison, we ran a Metropolis Hastings Markov chain Monte Carlo (MCMC) algorithm for one million steps using an initial zero-mean Gaussian proposal distribution. The proposal covariance is adapted at each step based on the Markov chain of the previous samples. We treated the first 500,000 steps of MCMC as burn-in and extracted a thinned set of $25,000$ samples at uniformly spaced iterations. Figure~\ref{fig:LVmodel} compares the MCMC samples to $25,000$ samples generated using the entropic estimator $\widehat{T}^2(\cdot;x_1^*)_\sharp \rho_2$ for an observation $x_1^* \sim \mu(\cdot|x_2 = x_2^*)$ corresponding to the true parameter vector $x_2^* = (0.832, 0.041, 1.082, 0.040)$, which is displayed in red in Figure~\ref{fig:LVmodel}. 

\section{Proofs from Section~\ref{sec:gaussian_theory_results}}\label{app:proofs}
\begin{proof}[Proof of \cref{prop:closedform}]
For general $\rho \defeq \cN(0,A)$ and $\mu \defeq \cN(m,B)$ the OT map has closed form \citep{gelbrich1990formula}:
\begin{align*}
    T_{\mathrm{B}}^{\rho\to\mu}(x) = m + A^{-1/2}(A^{1/2}BA^{1/2})^{1/2}A^{-1/2}x\,.
\end{align*}
Let $\rho_t = (A_t\cdot)_\sharp\rho = \cN(0,A_t^2)$ and similarly $\mu_t = \cN(A_t m,A_t\Sigma A_t)$. Then it holds that
\begin{align*}
    T_{\mathrm{B}}^{\rho_t \to \mu_t} = A_t m + A_t^{-1}(A_t^2\Sigma A_t^2)^{1/2}A_t^{-1}x\,.
\end{align*}
Using now \eqref{eq:otmap_A}, we have
\begin{align*}
    T_{t}(x) &= A_t^{-1}( A_t m + A_t^{-1}(A_t^2\Sigma A_t^2)^{1/2}A_t^{-1})(A_t x) = m + A_t^{-2}(A_t^2\Sigma A_t^2)^{1/2}x\,.
\end{align*}
We now compute $T_{\text{CB}}$. First note that the block lower Cholesky decomposition is
\begin{align*}
    \mathsf{L}_\Sigma = \begin{pmatrix}
        \Sigma_{11}^{1/2} & 0 \\
        \Sigma_{21}\Sigma_{11}^{-1/2}  & W^{1/2}
    \end{pmatrix}\,,
\end{align*}
where we recall
\begin{align*}
    \Sigma = \begin{pmatrix}
        \Sigma_{11} & \Sigma_{21}^\top \\
        \Sigma_{21}  & \Sigma_{22}
    \end{pmatrix}\,,
\end{align*}
and that $W = \Sigma_{22} - \Sigma_{21}\Sigma_{11}^{-1}\Sigma_{21}^\top$\,. Since this is block-lower triangular, this satisfies the criteria to be a conditional Brenier map.
\end{proof}

\begin{proof}[Proof of \cref{prop:stats_gaussian_t}]
    Since the error incurred from estimating the mean $m$ is negligible, we henceforth assume that $\rho_t \defeq \cN(0,A_t^2)$ and $\mu_t \defeq \cN(0,A_t\Sigma A_t)$. The optimal transport map from $\rho_t$ to $\mu_t$ is given by
    \begin{align*}
        T_t'(y) \defeq A_t^{-1}(A_t^2\Sigma A_t^2)^{1/2}A_t^{-1}y\,.
    \end{align*}
    Similarly, we define $\widehat{T}_t'(y) \defeq A_t^{-1}(A_t^2\widehat{\Sigma} A_t^2)^{1/2}A_t^{-1}y$. 
    The main expression we want to bound is then
\begin{align*}
    \E\|\widehat{T}_t - T_t\|^2_{L^2(\rho)} &= \E\|A_t^{-1} ( \widehat{T}_t' - T_t')\|^2_{L^2(\rho_t)} 
    \\
    &\leq \|A_t^{-1}\|_{\text{op}}^2   \E \|\widehat{T}_t' - T_t'\|^2_{L^2(\rho_t)} \\
    &\leq t^{-1}\E \|\widehat{T}_t' - T_t'\|^2_{L^2(\rho_t)}\,.
\end{align*}
To continue, we note that
\begin{align}\label{eq:upper_lower_eigenvals}
    \sqrt{\lambda_{\min}(\Sigma)}I \preceq D T_t'(y) \preceq \sqrt{\lambda_{\max}(\Sigma)}I\,,
\end{align}
where $\lambda_{\max}(\Sigma) > 0$ (resp. $\lambda_{\min}(\Sigma) > 0$) is the largest (resp. smallest) eigenvalue of the covariance matrix $\Sigma$. \Cref{eq:upper_lower_eigenvals} follows from Theorem 5 of \citet{chewi2022entropic}, where we compute $\nabla^2(-\log\rho_t) = A_t^{-2}$ and $(\lambda_{\max}(\Sigma))^{-1}A_t^{-2} \preceq \nabla^2(-\log\mu_t) \preceq (\lambda_{\min}(\Sigma))^{-1}A_t^{-2}$. By a direct application of a smoothness result by \citet[][Theorem 6]{manole2021plugin}, we have
\begin{align*}
    \E\|\widehat{T}_t' - T_t'\|^2_{L^2(\rho_t)} 
    \leq \kappa(\Sigma)\E W_2^2(\hat{\mu}_t,\mu_t)\,,
\end{align*}
where $\kappa(\Sigma) = \lambda_{\max}(\Sigma)/\lambda_{\min}(\Sigma)$, and $\hat{\mu}_t \defeq \cN(0,A_t\widehat{\Sigma}A_t)$. Note that the remaining Wasserstein distance term can be upper bounded (since $\|A_t\|_{\text{op}}\leq 1$) as
\begin{align*}
    \E W_2^2(\hat{\mu}_t,\mu_t) &\leq \E \int \|A_t(\widehat{\Sigma}^{1/2} - \Sigma^{1/2})(x)\|^2 \dd \rho(x) \leq \E \|\widehat{T} - T\|^2_{L^2(\rho)} \lesssim_{\Sigma} n^{-1}\,,
\end{align*}
where $T(x) \defeq \Sigma^{1/2}x$ and similarly $\widehat{T}(x) \defeq \widehat{\Sigma}^{1/2}x$, and the final inequality is well-known; see \cite{flamary2019concentration,divol2022optimal}.

Putting everything together, we have
\begin{align*}
    \E\|\widehat{T}_t - T_t\|^2_{L^2(\rho)} \lesssim_{\Sigma} t^{-1}n^{-1}\,.
\end{align*}
\end{proof}

\begin{proof}[Proof of \cref{prop:quantitative_gaussian_result}]
We write $Q_t \defeq A_t^2$. Suppose that for $t$ small enough, $M_t$ is a positive-definite matrix such that both 
\begin{align}\label{eq:main_gaussian_proof}
    \| M_t^2 - (Q_t\Sigma Q_t)\|_{\text{op}} \lesssim t^3\,, \quad \text{and} \quad \|Q_t^{-1}M_t - \mathsf{L}_\Sigma\|_{\text{op}}\lesssim t\,.
\end{align}
The first inequality implies that
\begin{align*}
    M_t^2 \preceq (Q_t \Sigma Q_t) + c t^3 I\,.
\end{align*}
Since square-roots preserves the positive-definite order, an appropriate Taylor expansion 
gives
\begin{align*}
    M_t \preceq (Q_t \Sigma Q_t)^{1/2} + \frac{c}{2}t^3 (Q_t \Sigma Q_t)^{-1/2} \preceq \frac{c}{2}t^2\Sigma^{-1/2}\,.
\end{align*}
Pre-multiplying by $Q_t^{-1} \preceq t^{-1}I$, we have that
\begin{align*}
    Q_t^{-1}M_t - Q_t^{-1}(Q_t\Sigma Q_t)^{1/2} \preceq \frac{c}{2}t \Sigma^{-1/2}\,,
\end{align*}
which implies $\| Q_t^{-1}M_t - Q_t^{-1}(Q_t\Sigma Q_t)^{1/2} \|_{\text{op}} \leq \frac{c}{2}t (\lambda_{\min}(\Sigma))^{-1/2}$. Combined with the second inequality in \eqref{eq:main_gaussian_proof} via triangle inequality, we arrive at
\begin{align*}
    \| Q_t^{-1}(Q_t\Sigma Q_t)^{1/2}  - \mathsf{L}_\Sigma\|_{\text{op}}\lesssim_{\Sigma} t\,.
\end{align*}
Since $\|\cdot\|_{\text{F}}^2 \leq d \|\cdot\|_{\text{op}}^2$\,, the proof is complete if we can find such an $M_t$. The matrix outlined in~\cref{lem:sqrt_perturbation} satisfies the inequalities in \eqref{eq:main_gaussian_proof}, which is the one we take to complete the proof.
\end{proof}

\begin{proof}[Proof of \cref{prop:entropic_bias}]
For the linear map $x\mapsto Sx$, write $\rho_S \defeq (S\cdot)_\sharp\rho$ and $\mu_S \defeq (S\cdot)_\sharp\mu$. Specifically, $\rho_S = \cN(0,SAS)$ and $\mu_S = \cN(b,SBS)$. The closed-form expression for $\tilde{T}_{\eps,S}$, the entropic map from $\rho_S$ to $\mu_S$ for $\eps \geq 0$ can be computed as
\begin{align*}
    \tilde{T}_{\eps,S}(y) &= ((SAS)^{-1/2}M_{\eps,S}^{1/2}(SAS)^{-1/2} - \tfrac{\eps}{2}(SAS)^{-1})y + Sb\,, \\
    &= ((SAS)^{-1/2}[M_{\eps,S}^{1/2} - \tfrac{\eps}{2}I](SAS)^{-1/2})y + Sb
\end{align*}
where $M_{\eps,S} = (SAS)^{1/2}SBS(SAS)^{1/2} + \tfrac{\eps^2}{4}I$; see~\citet[Proposition 3]{pooladian2022debiaser}. The expression for the Brenier map between $\rho_S$ and $\mu_S$, $\tilde{T}_{0,S}(y)$, follows the same formula. Using \eqref{eq:otmap_A}, the maps between $\rho$ and $\mu$ under the cost $\tfrac12\|S(x-y)\|^2_2$ are given by\looseness-1
\begin{align*}
    & T_{\eps,S}(x) = S^{-1}(SAS)^{-1/2}[M_{\eps,S}^{1/2}  - \tfrac{\eps}{2}I](SAS)^{-1/2}Sx + b 
\end{align*}
We can now directly compute
\begin{align*}
    \|T_{\eps,S} - T_{0,S}\|^2_{L^2(\rho)} &= \tr\bigl[\bigl( S^{-1}(SAS)^{1/2}[M_{\eps,S}^{1/2} - M_{0,S}^{1/2} - \tfrac{\eps}{2}I](SAS)^{-1/2}S\bigr)^2A\bigr] \\
    &= \tr\bigl[\bigl((SAS)^{-2}[M_{\eps,S}^{1/2} - M_{0,S}^{1/2} - \tfrac{\eps}{2}I]^2\bigr)A\bigr] \\
    &= \tr\bigl[S^{-2}A^{-1}S^{-2}[M_{\eps,S}^{1/2} - M_{0,S}^{1/2} - \tfrac{\eps}{2}I]^2\bigr]\,,
\end{align*}
where we use the fact that all matrices are symmetric, and thus commute under trace. A direct application of \citet[Lemma 2]{pooladian2022debiaser} tells us that
\begin{align*}
    M_{\eps,S}^{1/2} = M_{0,S}^{1/2} + \frac{\eps^2}{8}M_{0,S}^{-1/2} + O(\eps^4)\,.
\end{align*}
Expanding the square inside the trace and applying the above Taylor expansion, we see that many terms cancel, and, dropping the remaining negative-definite terms, one arrives at 
\begin{align*}
    \|T_{\eps,S} - T_{0,S}\|^2_{L^2(\rho)} \leq \frac{\eps^2}{4}\tr(S^{-4}A^{-1}) \leq \frac{\eps^2}{4}\|S^{-4}\|_{\text{op}}I_0(\rho)\,,
\end{align*}
where recall $I_0(\rho) = \tr(A^{-1})$ is the Fisher information of $\rho = \cN(0,A)$.
\end{proof}

\begin{proof}[Proof of \cref{thm:entropic_gaussian}]
Writing the expansion \eqref{eq:approx_entropic} and using both \cref{eq:quantitative_gaussian_result} and \cref{prop:entropic_bias}, we obtain (using that $I_0(\rho) = d$)
\begin{align*}
    \|T_{\eps,t} - T_{\text{CB}}\|^2_{L^2(\rho)} &\leq \|T_{\eps,t} - T_{t}\|^2_{L^2(\rho)} + \|T_{t} - T_{\text{CB}}\|^2_{L^2(\rho)} \\
    &\lesssim_{\Sigma} 2 \frac{\eps^2}{4}\tr(A_t^{-4} I) + 2t^2 \\
    &\leq \frac{\eps^2}{2}dt^{-2} + 2t^2\\
    &\lesssim_{\Sigma,d} t^2\,,
\end{align*}
where we choose $\eps^* \asymp_{\Sigma,d} t^{2}$ in the penultimate inequality to conclude. 
\end{proof}

\begin{lemma}[Square-root perturbation]\label{lem:sqrt_perturbation}
Consider the assumptions in \cref{prop:quantitative_gaussian_result}. There exists a matrix $M$ such that the following inequalities hold
\begin{align}\label{eq:sqrt_perturbation}
    \| M^2 - (Q_t\Sigma Q_t)\|_{\text{op}} \lesssim &t^3\,\\
    \|Q_t^{-1}M - \mathsf{L}_\Sigma\|_{\text{op}}\lesssim &t\,,
\end{align}
with $Q_t = A_t^2$.
\end{lemma}
\begin{proof}
First, recall that 
\begin{align*}
    \Sigma = \begin{pmatrix}
        \Sigma_{11} & \Sigma_{21}^\top \\
        \Sigma_{21}  & \Sigma_{22}
    \end{pmatrix}\,, \quad \mathsf{L}_\Sigma = \begin{pmatrix}
        \Sigma_{11}^{1/2} & 0 \\
        \Sigma_{21}\Sigma_{11}^{-1/2}  & S^{1/2}
    \end{pmatrix}\,,
\end{align*}
where $S = \Sigma_{22} - \Sigma_{21}\Sigma_{11}^{-1}\Sigma_{21}^\top$\,.

Let
\begin{align}
M\defeq \begin{pmatrix}
        \Sigma_{11}^{1/2} & 0 \\
        0  & 0
\end{pmatrix} + 
    t \begin{pmatrix}
        0 & \Sigma_{11}^{-1/2}\Sigma_{21}^\top\\
        \Sigma_{21}\Sigma_{11}^{-1/2}  & S^{1/2}
    \end{pmatrix} + 
    t^2 \begin{pmatrix}
        B_{11} & B_{21}^\top \\
        B_{21}  & 0
    \end{pmatrix}\,. 
    \end{align}
We compute
\begin{align*}
    M^2 
    &= (Q_t \Sigma Q_t) + t^2 \begin{pmatrix}
        \Sigma_{11}^{1/2}B_{11} + B_{11}\Sigma_{11}^{1/2} + \Sigma_{11}^{-1/2}\Sigma_{21}^\top\Sigma_{21}\Sigma_{11}^{-1/2} & \Sigma_{11}^{1/2}B_{21}^\top + \Sigma_{11}^{-1/2}\Sigma_{21}^\top\Sigma_{11}^{-1/2} \\
        B_{21}\Sigma_{11}^{1/2} + \Sigma_{11}^{-1/2}\Sigma_{21}\Sigma_{11}^{-1/2}  & 0
    \end{pmatrix}
    \\
    &\qquad + O(t^3) \,,
\end{align*}
where 
\begin{align*}
    (Q_t \Sigma Q_t) = 
    \begin{pmatrix}
        \Sigma_{11} & t\Sigma_{21}^\top \\
        t\Sigma_{21}  & t^2 \Sigma_{22}
\end{pmatrix}\,.
\end{align*}
We choose $B_{11}$ and $B_{21}$  to satisfy
\begin{align*}
    & \Sigma^{1/2}_{11} B_{11} + B_{11}\Sigma_{11}^{1/2} = -\Sigma_{11}^{1/2}\Sigma_{21}^\top\Sigma_{21}\Sigma_{11}^{-1/2} \\
    &B_{21} = -\Sigma_{11}^{-1}\Sigma_{21}^\top S^{1/2}\,.
\end{align*}
Thus, $\| M^2  - (Q_t \Sigma Q_t)\|_{\text{op}} \lesssim t^3$, which is the first claim. For the second, note that
\begin{align*}
    Q_t^{-1} M = \begin{pmatrix}
        \Sigma_{11} + O(t)& t\Sigma_{21}^\top + O(t)\\
        \Sigma_{21}\Sigma_{11}^{-1/2}  & S^{1/2} + O(t) 
\end{pmatrix} = \mathsf{L}_\Sigma + 
\begin{pmatrix}
       O(t)& t\Sigma_{21}^\top + O(t)\\
        0 & O(t) 
\end{pmatrix}\,,
\end{align*}
and thus $\| Q_t^{-1}M - \mathsf{L}_\Sigma\|_{\text{op}} \lesssim t$.
\end{proof}

\end{document}